\documentclass[sigconf]{acmart}

\usepackage{booktabs} % For formal tables
\usepackage[ruled]{algorithm2e}

\begin{document}
\title{Hierarchical Clustering with Prior Knowledge}
% \titlenote{Produces the permission block, and copyright information}
% \subtitle{}
% \subtitlenote{The full version of the author's guide is available as \texttt{acmart.pdf} document}

\author{Xiaofei Ma}
% \authornote{}
\orcid{}
\affiliation{%
  \institution{Amazon.com Inc.}
  \streetaddress{300 Boren Ave N.}
  \city{Seattle}
  \state{Washington}
  \postcode{98109}
}
\email{xiaofeim@amazon.com}

\author{Satya Dhavala}
% \authornote{The secretary disavows any knowledge of this author's actions.}
\affiliation{%
  \institution{Amazon.com Inc.}
  \streetaddress{300 Boren Ave N.}
  \city{Seattle}
  \state{Washington}
  \postcode{98109}
}
\email{sdhavala@amazon.com}

% \renewcommand{\shortauthors}{B. Trovato et al.}

\begin{abstract}
Hierarchical clustering is a class of algorithms that seeks to build a hierarchy of clusters.
It has been the dominant approach to constructing embedded classification schemes since it outputs dendrograms, 
which capture the hierarchical relationship among members at all levels of granularity, simultaneously.  
Being greedy in the algorithmic sense, 
a hierarchical clustering partitions data at every step solely based on a similarity / dissimilarity measure.
The clustering results oftentimes depend on not only the distribution of the underlying data, 
but also the choice of dissimilarity measure and the clustering algorithm.
In this paper, we propose a method to incorporate prior domain knowledge about entity relationship into the hierarchical clustering.
Specifically, we use a distance function in ultrametric space to encode the external ontological information.
We show that popular linkage-based algorithms can faithfully recover the encoded structure.
Similar to some regularized machine learning techniques,
we add this distance as a penalty term to the original pairwise distance to regulate the final structure of the dendrogram.
As a case study, we applied this method on real data in the building of a customer behavior based product taxonomy for an Amazon service,
leveraging the information from a larger Amazon-wide browse structure.
The method is useful when one want to leverage the relational information from external sources,
or the data used to generate the distance matrix is noisy and sparse.
Our work falls in the category of semi-supervised or constrained clustering.
\end{abstract}





\maketitle

\section{Introduction}
\label{sec: intro}

Hierarchical clustering is a a prominent class of clustering algorithms.
It has been the dominant approach to constructing embedded classification schemes \cite{Murtagh2012}. 
Compared with partition-based methods (flat clustering) such as K-means,
a hierarchical clustering offers several advantages.
First, there is no need to pre-specify the number of clusters. 
Hierarchical clustering outputs dendrogram (tree), 
which the user can then traverse to obtain the desired clustering.
Second, the dendrogram structure provides a convenient way of exploring entity relationships at all levels of granularity.  
Because of that, for some applications such as taxonomy building,
the dendrogram itself, not any clustering found in it, is the desired outcome.
For example, hierarchical clustering has been widely employed and explored within the context of phylogenetics,
which aims to discover the relationships among individual species, 
and reconstruct the tree of biological  evolution.
Furthermore, when dataset exhibits multi-scale structure,
hierarchical clustering is able to generate a hierarchical partition of the data at different levels of granularity,
while any standard partition-based algorithm will fail to capture the nested data structure.
 
In a typical hierarchical clustering problem, 
the input is a set of data points and a notion of dissimilarity between the points,
which can also be represented as a weighted graph whose vertices are data points,
and edge weights represent pairwise dissimilarities between the points.
The output of the clustering is a dendrogram, 
a rooted tree where each leaf node represents a data point,
and each internal node represents a cluster containing its descendant leaves.
As the internal nodes get deeper in the tree, 
the points within the clusters become more similar to each other, 
and the clusters become more refined.
Algorithms for hierarchical clustering generally fall into two types:
Agglomerative (\textquotedblleft bottom up\textquotedblright) approach: 
each observation starts in its own cluster, at every step a pair of most similar clusters are merged.
Divisive (\textquotedblleft top down\textquotedblright) approach: 
all observations start in one cluster, and splits are performed recursively, 
dividing a cluster into two clusters that will be further divided.

As a popular data analysis method, hierarchical clustering has been studied and used for decades.
Despite its widespread use, it has rather been studied at a more procedural level in terms of practical algorithms.
There are many hierarchical algorithms.
Oftentimes, different algorithms produce dramatically different results on the same dataset.
Compared with partition-based methods such as K-means and K-medians,
hierarchical clustering has a relatively underdeveloped theoretical foundation.
Very recently, Dasgupta \cite{Dasgupta2015} introduced an objective function for hierarchical clustering,
and justified it for several simple and canonical situations.
A theoretical guarantee for this objective was further established \cite{Moseley2017} on some of the widely used hierarchical clustering algorithms.
Their works give insight into what those popular algorithms are optimizing for.
Another route of theoretical research is to study the clustering schemes under an axiomatic view \cite{Meila2005, Carlsson2013, Eldridge2015, Zadeh2009, Ackerman2010, Ben-David2009},
charactering different algorithms by the significant properties they satisfy.  
One of the influential works is Kleinberg's impossibility theorem \cite{Kleinberg2002},
where he proposed three axioms for partitional clustering algorithms, 
namely scale-invariance, richness and consistency.
He proved that no clustering function can simultaneously satisfy all three.
It is showed \cite{Carlsson2010}, however, if a nested family of partitions instead of fixed single partition is allowed, 
which is the case for hierarchical clustering,
single linkage hierarchical clustering is the unique algorithm satisfying the properties.
The stability and convergence theorems for single link algorithm are further established.
Ackerman \cite{Ackerman2016} proposed two more desirable properties, namely, locality and outer consistency,
and showed that all linkage-based hierarchical algorithms satisfy the properties.
Those property-based analyses provide a better understanding of the techniques,
and guide users in choosing algorithms for their crucial tasks.

Based on similarity information alone, 
clustering is inherently an ill-posed problem where the goal is to partition the data into some unknown number of clusters
so that within cluster similarity is maximized while between cluster similarity is minimized \cite{Jain2010}.
It's very hard for a clustering algorithm to recover the data partitions that satisfies various criteria of a concrete task.
Therefore, any external or side information from other sources can be extremely useful in guiding clustering solutions.
Clustering algorithms that leverage external information fall into the category of semi-supervised or constrained clustering \cite{Basu2008}.
There are many ways to incorporate external information \cite{Bair2013, Wagstaff2001, Liu2007, Xing1986}.
Starting from instance-level constraints such as must-link constraints and cannot-link constraints,  
many approaches try to modify the objective function of the algorithms to incorporate pairwise constraints.
Beyond pairwise constraints, external knowledge has been used as the seeds for clustering,
cluster size constraints, or as prior probabilities of cluster assignment.
However, the majority of existing semi-supervised clustering methods are based on partition-based clustering. 
Comparatively few methods on hierarchical clustering have been proposed.
In fact, human is very good at summarizing and extracting high level relational information between entities.
Human built taxonomies, such as WordNet, Wikipedia, 20 newsgroup dataset etc. 
are high quality sources of ontological information that a hierarchical clustering algorithm can leverage.
Several factors contributed to the underdevelopment in the semi-supervised hierarchical clustering algorithms.
One is the lack of global objective functions.
Only very recently an objective function for hierarchical clustering was proposed \cite{Dasgupta2015}.
Another reason is that simple must-link and cannot-link constraints used in flat clustering are not suitable in
hierarchical clustering since entities are linked at different level of granularity.
Furthermore, the output of hierarchical clustering is a dendrogram which is harder to represent than the result from a flat clustering.

In this paper, we focus on agglomerative hierarchical clustering algorithms since divisive algorithms can be considered as 
a repeated partitional clustering (bisectioning).
We describe a method of incorporating prior ontological knowledge into agglomerative hierarchical clustering
by using a distance function in ultrametric space representing the complete or partial tree structure.
The constructed ultrametic distance is combined with the original task-specific distance to form a new distance measure between the data points.
The weight between the two distance components, which reflects the confidence of prior knowledge,
is a hyper-parameter that can be tuned in a cross-validation manner, 
by optimizing an external task-specific metric.
We then use a property-based approach to select algorithms to solve the semi-supervised clustering problem.
We note that there are several pioneer works on constrained hierarchical clustering \cite{Huang2016, Heller2005, Bade2007, Jose2000}.
Davidson \cite{Davidson2005} explored the feasibility problem of incorporating 4 different instance and cluster level constraints into hierarchical clustering.
Zhao \cite{Zhao2010} studied hierarchical clustering with order constraints in order to capture the ontological information.
Zheng \cite{Zheng2011} represented triple-wise relative constraints in a matrix form, 
and obtained the ultrametric representation by solving a constrained optimization problem.
Compared with previous studies, our goal is to recover the hierarchical structure of the data 
which resembles existing ontology and yet provides new insight into entity relationships based on a task-specific distance measure.
The external ontological knowledge serves as soft constraints in our approach,
which is different from the hard constraints used in the previous works.
Our constructed distance measure also fits naturally with the global object function \cite{Dasgupta2015} recently proposed for hierarchical clustering.

The paper is organized as follows: 
In Section \ref{sec: setting}, we state the problem and introduce the concepts used in this paper.
In Section \ref{sec: method}, we discuss our approach to solving the semi-supervised hierarchical clustering problem.
In Section \ref{sec: experiment}, we present a case study of applying the proposed method on real data
to the building of a customer behavior based product taxonomy for an Amazon service. 
Finally, we summarize the results in Section \ref{sec: conclusion}.

\section{Problem Setting}
\label{sec: setting}
In this section we define the context and the problem we want to solve, 
i.e. the semi-supervised hierarchical clustering problem.

Given a set of data points $X = \{ x_{1}, x_{2}, ..., x_{n} \} $,
a pairwise dissimilarity measure $D = \{ d(x_{i}, x_{j})  |  x_{i}, x_{j}   \in X \}$,
a task-specific performance measure $\mu$,
and a complete or partial tree structure $T$ contains external ontological information, 
whose leaf nodes are instances of $X$, 
and whose internal nodes are clusters containing descendant leaves.
The goal of the semi-supervised hierarchical clustering problem is to output a dendrogram over $X$
represented as a pair $(X, \theta)$, 
where $X$ is the set of data points, 
$\theta : [ \, 0,  \infty) \to \wp(X)$, 
$\wp(X)$ is a partition of $X$,
such that the dendrogram $\theta$ resembles $T$ and performs best in terms of $\mu$.

Two important concepts related to the above problem setting are the notion of dissimilarity and the dendrogram. 

\begin{definition}
A dissimilarity measure $D$ is usually represented as a pair $(X, d)$,
where X is a set, and $d : X \times X \to \mathbf{R}^{+} $ such that for any $x_{i}, x_{j} \in X$:
\begin{align*}
	&\text{$1.$ $d(x_{i}, x_{j}) \geq 0$, non-negativity}\\
	&\text{$2.$ $d(x_{i}, x_{j}) = 0$ if and only if $i = j$, identity} \\
	&\text{$3.$ $d(x_{i}, x_{j}) = d(x_{j}, x_{i})$, symmetry}
\end{align*}
\end{definition}
As an example, cosine dissimilarity is a commonly used dissimilarity measure in high-dimensional positive space.  

\begin{definition}
If the dissimilarity also satisfies the following triangle inequality,
for any $x_{i}, x_{j}, x_{k} \in X$:
\begin{equation}
	\begin{aligned}	
	d(x_{i}, x_{k}) \leq d(x_{i}, x_{j}) + d(x_{j}, x_{k})
	\end{aligned}
	\label{eq: triangle}
\end{equation}
then we have a distance measurement in the metric space.
\end{definition}
Euclidean distance and manhattan distance are popular metric space distances.

\begin{definition}
A dendrogram $\theta$ is a tree that satisfies the following conditions \cite{Carlsson2010}:
\begin{align*}
	&\text{$1.$ $\theta(0) = \{ \{x_{1} \}, ..., \{x_{n} \} \}$} \\
	&\text{$2.$ There exists $t_{0}$ such that $\theta (t)$ contains only one cluster for $t  \geq t_{0}$.} \\
	&\text{$3.$ If $r \leq s$, then $\theta (r)$ refines $\theta (s).$} \\
	&\text{$4.$ For all $r$, there exists $\epsilon > 0$ such that $\theta(r) = \theta(t) $ for $t \in [\,r, r + \epsilon]\,$.}
\end{align*}
\end{definition}
Condition 1 ensures that the initial partition is the finest possible, each data point forms a cluster.
Condition 2 tells that for large enough $t$, the partition becomes trivial.
The whole space is one cluster.
Condition 3 ensures that the structure of dendrogram is nested.
Condition 4 requires that the partition is stable under small perturbation of size $\epsilon$.
The parameter of dendrogram $\theta$ is a measure of scale, 
and reflected in the height of different levels.
The notion of resemblance between dendrograms will be discussed more in Section \ref{subsec:property}.

\section{Proposed Method}
\label{sec: method}
In order to incorporate prior knowledge into hierarchical clustering,
we need a way to faithfully represent prior relational information between entities.
Since relational knowledge such as hyponymy and synonymy relations in WordNet, 
class taxonomy in the 20-newsgroups dataset can usually be represented as a tree,
this suggests that it is convenient to define a distance function that leverages the tree structure.
In fact, Resnik's approach \cite{Resnik1995} to semantic similarity between words was the first attempt to brings together 
the ontological information in WordNet with the corpus information.
Figure \ref{fig:wordnet} shows a fragment of structured lexicons defined in WordNet.

\begin{figure}[h]
	\centering
	%\begin{minipage}{0.8\textwidth}
		\includegraphics[width=0.45\textwidth]{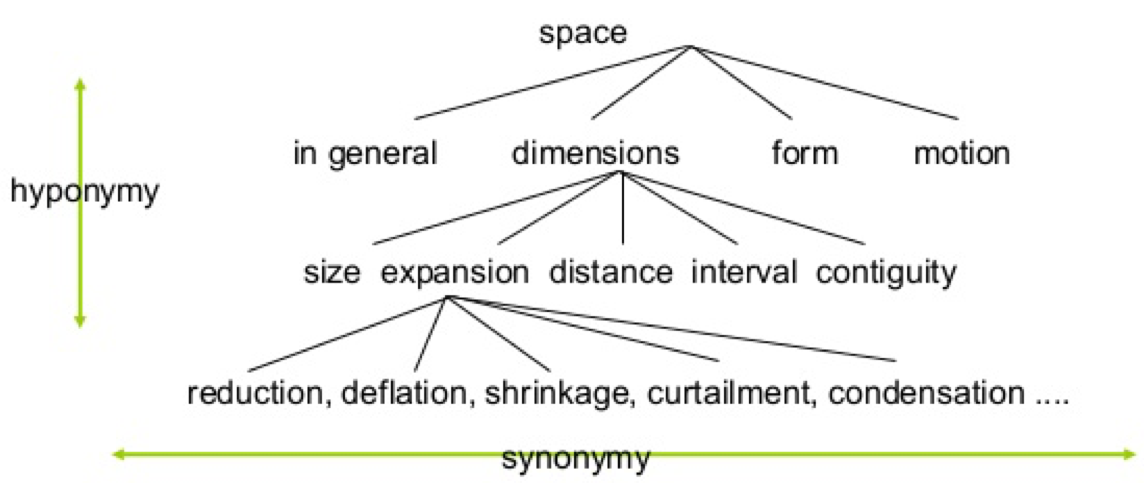}
		\caption{\texttt{Structured lexicons from WordNet}}
		\label{fig:wordnet}
	%\end{minipage}
\end{figure}

To encode a tree structure, we first introduce the concept of ultrametric space.
\begin{definition}
A metric space is an ultrametric $(X, u)$ if and only if, 
\begin{equation}
	\begin{aligned}
		d(x_{i}, x_{k}) \leq max(d(x_{i}, x_{j}), d(x_{j}, x_{k}))
	\end{aligned}
	\label{eq:ultrametric}
\end{equation}
\end{definition}
The ultrametric condition requires that every triangle formed by any three data points 
has to be an acute isosceles triangle,
which is a stronger condition than the triangle inequality in Equation \ref{eq: triangle}.

It is well known that a dendrograms can be represented as ultrametrics.
The relationship between dendrograms and ultrametric has been discussed in 
several works \cite{Roy2016, Jardine1971, Hartigan1985, Jain1988}.
The equivalence between dendrograms and ultrametrics was further established by Carlsson in \cite{Carlsson2010}.
A hierarchical clustering algorithm essentially outputs a map from finite metric space $(X, d)$ into finite ultrametric space $(X, u)$.

\subsection{An ultrametic function to encode prior relational information}
\label{subsec: ultrametric_function}
We now propose an ultrametric distance function to encode the tree structure between entities.
\begin{definition}
\label{definition: distance}
Let $T$ be a rooted tree of entity relationship.
For any node $v$ in $T$, let $T[ \, v \, ]$ be a subtree rooted at $v$, 
$leaves(T[ \, v \, ])$ be the leaves of the subtree,
and $|leaves(T[ \, v \, ])|$ be the number of leaf nodes.
For any leaf node $x_{i}, x_{j}$, the expression $x_{i} \lor x_{j}$ denotes their lowest common ancestor in T. 
We define a distance function between any leaf node $x_{i}, x_{j}$ as follows:
\begin{equation}
	\begin{aligned}
		u_{T}(x_{i}, x_{j}) = | leaves(T[ \, x_{i} \lor x_{j} \, ])| / |leaves(T[ \, root \, ] )|
	\end{aligned}
	\label{eq: distance}
\end{equation}
\end{definition}
In the above definition, $|leaves(T[ \, root \, ] )|$ is the total number of leaf nodes in the tree $T$.
It is a normalization constant to ensure the distance is between $[0, 1]$.
As an example, Figure \ref{fig: treesample} shows a small tree consisting of $6$ leaf nodes.
According to Definition \ref{definition: distance},
the distances between pairs $\{1, 2\}, \{1, 3\}, \{1, 4\}, \{1, 5\}, \{1, 6\}$
are $2/6, 4/6, 4/6, 6/6$ and $6/6$, respectively.
Although the tree structure in Figure \ref{fig: treesample} doesn't specify the exact distance values,
it encodes the hierarchical relations between data points.
It is easy to see that point $1$ is more similar to point $2$ than to point $3$ or $5$.

\begin{figure}[h]
	\centering
	%\begin{minipage}{0.8\textwidth}
		\includegraphics[width=0.4\textwidth]{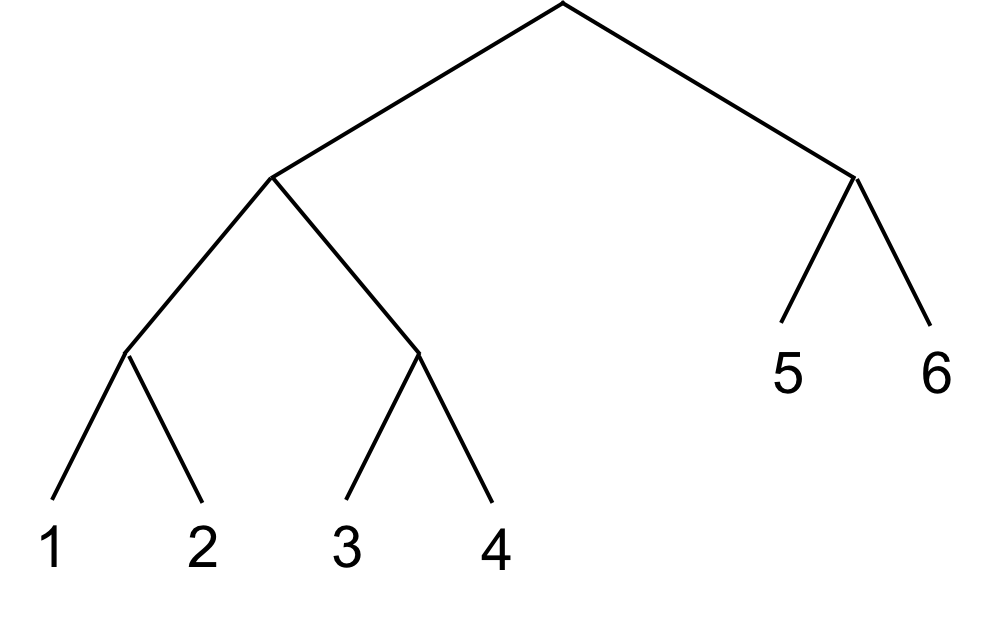}
		\caption{\texttt{A small tree of 6 leaf nodes}}
		\label{fig: treesample}
	%\end{minipage}
\end{figure}

\begin{lemma}
The distance function defined in Definition \ref{definition: distance} is an ultrametric.
\end{lemma}
\begin{proof}
For any leaf node $x_{i}, x_{j}, x_{k}$ in $T$,
$x_{k}$ is either in the subtree $T[ \, x_{i} \lor x_{j} \, ]$ or not in the subtree $T[ \, x_{i} \lor x_{j} \, ]$.
If $x_{k}$ is in the subtree $T[ \, x_{i} \lor x_{j} \, ]$, then $u_{T}(x_{i}, x_{k}) \leq u_{T}(x_{i}, x_{j})$.
If $x_{k}$ is not in the subtree $T[ \, x_{i} \lor x_{j} \, ]$, 
we have $T[ \, x_{i} \lor x_{k} \, ] = T[ \, x_{j} \lor x_{k} \, ] = T[ \, (x_{i} \lor x_{j}) \lor x_{k}\, ]$,
then $u_{T}(x_{i}, x_{k}) = u_{T}(x_{j}, x_{k})$.
Therefore, in either case, $u_{T}(x_{i}, x_{k}) \leq max(u_{T}(x_{i}, x_{j}), u_{T}(x_{j}, x_{k}))$.
\end{proof}
Because of the equivalence between dendrograms and ultrametrics,
once we encode the tree using an ultrametric distance,
there is a unique dendrogram corresponding to it.

A pairwise distance function quantifies the dissimlarity between any pair of points.
However, it doesn't define the distance between clusters of points. 

Linkage-based hierarchical clustering algorithms calculate distance between clusters based on different heuristics.
Let $\ell(C, C', d)$ be a linkage function that assigns a non-negative value to each pair of non-empty clusters $\{C, C'\}$
based on a pairwise distance function $d$.
Some choices of linkage functions are:
\begin{align*}
	&\text{$1.$ $\ell_{SL}(C, C', d) = {\underset{x \in C, x' \in C'} {min} d(x, x')} $, single linkage}\\
	&\text{$2.$ $\ell_{CL}(C, C', d) = {\underset{x \in C, x' \in C'} {max} d(x, x')} $, complete linkage} \\
	&\text{$3.$ $\ell_{AL}(C, C', d) = {\underset{x \in C, x' \in C'} {\Sigma} d(x, x')} / (|C| \cdot |C'|)$, average linkage}
\end{align*}
All three linkage functions lead to a popular hierarchical clustering algorithm.
However, it is known that the results from average link and complete link algorithms depend on the ordering of points,
while single link is exempted from this undesirable feature.
The cause lies in the way that an algorithm deals with situation when more than two points are equally good candidates for merging next.
Since we merge the data points two at a time, then the merge order will determine the final structure of the dendrogram.
However, it can be shown that when an ultrametric distance is used, all three linkage-based algorithms will output the same dendrogram.
\begin{theorem}
The dendrogram structure from a complete linkage or average linkage hierarchical algorithm 
is independent of the merge order of equally good candidates when the distance measure is an ultrametric. 
(The proof is in Appendix \ref{appendix: completelink}.)
\end{theorem}

As an example, for the small tree defined in Figure \ref{fig: treesample} and the distance function defined in Equation \ref{eq: distance}, 
all three linkage-based algorithms produce the same dendrogram presented in Figure \ref{fig: dendrogram}.
The dendrogram faithfully encodes all the grouping relations between leaf nodes from the original tree.

\begin{figure}[h]
	\centering
	%\begin{minipage}{0.8\textwidth}
		\includegraphics[width=0.4\textwidth]{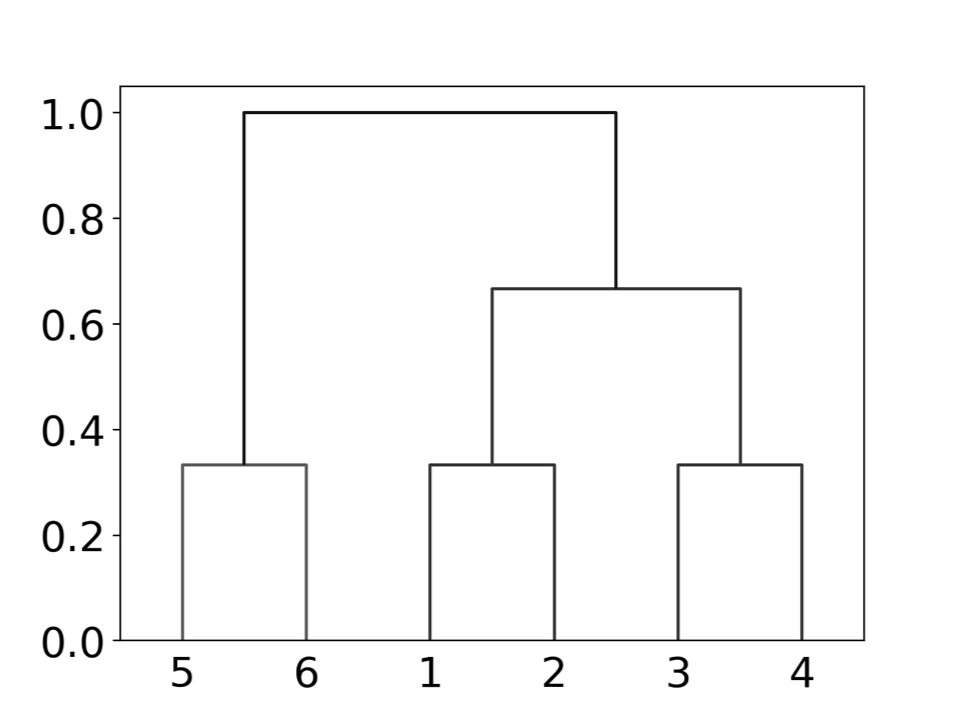}
		\caption{\texttt{Dendrogram of the example tree}}
		\label{fig: dendrogram}
	%\end{minipage}
\end{figure}  
 
\subsection{Combine the two distance components}
\label{subsec:combine}
To incorporate the external ontological information into the hierarchical clustering, 
we combine the as-defined ultrametric distance function with the problem-specific distance measure using a weighted sum of the two components.
Let $d_{P}$ be the problem-specific distance
(we normalize it so that its value is between $[0, 1]$), 
and $u_{T}$ be the ultrametric distance encoding the prior ontological knowledge.
The new distance function to be fed into a hierarchical clustering algorithm can be constructed as follows:
\begin{equation}
	\begin{aligned}
		d(x_{i}, x_{j}) = (1-\alpha) \cdot d_{P}(x_{i}, x_{j}) + \alpha \cdot u_{T}(x_{i}, x_{j}) 
	\end{aligned}
	\label{eq:newdistance}
\end{equation}
Similar to some regularized machine learning techniques,
the ultrametric distance is added as a penalty term to the original pairwise distance.
When $\alpha = 0$,
we go back to the unregulated hierarchical clustering case,
in which only the problem-specific distance is used.
When $\alpha = 1$, we recover the relational structure from the external source.
Essentially, $\alpha \cdot u_{T}(x_{i}, x_{j})$ measures the minimal effort that a hierarchical clustering algorithm needs to make in order to join $x_{i}$ and $x_{j}$.
The hyper-parameter $\alpha$ determines the proportion that the prior knowledge contribute to the clustering.
It reflects our confidence in each component.
Since the ultrametric term is added pair-wisely, 
the new distance function fits naturally with the global object function proposed in \cite{Dasgupta2015}.
In that context, the ultrametic term is a soft constraint added to the object function. 

The tuning of the hyper-parameter $\alpha$ can be achieved in different ways depending on the availability of external labels or performance metric.
Without external gold standard, the tuning can be conducted by maximizing some internal quality measures such as Davies-Bouldin index or Dunn index. 
With the availability of external labels, parameter $\alpha$ can be tuned in a cross-validation manner.
Various performance measures have been proposed to evaluate clustering results given a gold standard \cite{Zhao2002}.
It should be noted that some performance metrics require conversion of a dendrogram into a flat partition.
In those cases, the number of the clusters $K$ is also hyper-parameter to tune.
If the dendrogram itself, not any clustering found in it, is the desired outcome, 
we can aggregate the performance metric across different $K$ for a given $\alpha$, 
and choose the dendrogram corresponding to the $\alpha$ with the best overall performance.  

\subsection{Property based approach for clustering algorithm selection}
\label{subsec:property}
In Equation \ref{eq:newdistance}, the overall distance function is no longer ultrametric 
if the problem-specific distance $d_{P}$ is not ultrametric.
To remediate the problem, one could convert the problem-specific distance function into an ultrametric distance.
However, finding the closest ultrametric to a noisy metric data is $NP$-complete.
We also need to specify a measure of distortion between the original metric and the approximated ultrametric \cite{DiSumma2015}.
One could also try to feed the problem-specific distance into a hierarchical clustering algorithm, and let the algorithm output an ultrametric for us.  
In fact, it is shown in \cite{Carlsson2010} that single linkage hierarchical clustering produces ultrametric outputs exactly
as those from a maximal sub-dominant ultrametric construction,
which is a canonical construction from metric to ultrametric.

In addition to the above property, 
single linkage algorithm also enjoys other properties that are important to applications such as taxonomy building.
In \cite{Ackerman2016}, 
Ackerman shows that all linkage-based hierarchical algorithms satisfying the locality and outer consistency properties.
However, it is observed that both complete linkage and average linkage are not stable under small perturbation,
and not invariant under permutation of data label \cite{Carlsson2010}.   
It is shown that only single linkage algorithm is stable in the Gromov-Hausdorff sense
and has nice convergence property \cite{Eldridge2015} .
Gromov-Hausdorff distance measures how far two finite spaces are from being isometric.
The stability property is critical to our distance function defined in Equation \ref{eq:newdistance}
since we'd like a continuous map from metric spaces into dendrograms as we change the hyper-parameter $\alpha$.
Based on the stability property, 
we can define the structure resemblance discussed in the problem statement Section \ref{sec: setting}.
We'd like the dendrogram from our semi-supervised method to be similar to the dendrogram encoding prior domain knowledge
as measured by Gromov-Hausdorff distance.
It can be shown that for two dendrograms $u, u'$ generated from single linkage algorithm defined on the same data set $X$,     
their Gromov-Hausdorff distance is bounded above by the $L_{\infty}$ norm of the difference between two underlying metric spaces $d, d'$.  

One drawback of single linkage algorithm is that it is not sensitive to variations in the data density,
which can cause \textquotedblleft chaining effect\textquotedblright.
However, we believe that this \textquotedblleft chaining effect\textquotedblright is alleviated in our semi-supervised approach
since we use a prior tree to regulate the dendrogram structure from clustering.
Based on the above reasons, 
we choose to use single linkage algorithm to solve our semi-supervised hierarchical clustering problem.

\begin{algorithm}[t]
\SetAlgoLined
\KwIn{dataset $X = \{ x_{1}, x_{2}, ..., x_{n} \} $, external tree structure $T$ defined on $X$, task-specific performance metric $\mu$}
\KwOut{dendrogram $\theta : [ \, 0,  \infty) \to \wp(X)$ that performs best in terms of $\mu$}
Pre-partition $X$ into $k$ sub-clusters\;
\For{each sub-cluster}{
	Calculate task-specific pariwise distance $d_{P}(x_{i}, x_{j})$\;
	Calculate ultrametric distance $u_{T}(x_{i}, x_{j})$ based on tree $T$\;
	\For{each $(\alpha, K)$ on the search grid}{
		Build dendrogram using Single-Link\;
		Convert the dendrogram into $K$ flat partitons\;
		Evaluate performance metric $\mu$\;
		}
	Find optimal $\alpha$ for each sub-cluster by aggregating $\mu$ across different $K$\;
}
Combine $k$ sub-clusters into one dendrogram by Single-Link\
\caption{Semi-supervised hierarchical clustering}
\label{alg: one}
\end{algorithm}

\subsection{Computational complexity}
All agglomerative hierarchical clustering methods need to compute the distance between all pairs in the dataset.
The complexity of this step, in general, is $O(n^2)$, where $n$ is the number of data points.
In each of the subsequent $n-2$ merging iterations,
the algorithm needs to compute the distance between the most recently created cluster and all other existing clusters.
Therefore, the overall complexity is $O(n^3)$ if implemented naively.
If done more cleverly, the complexity can be reduced to $O(n^2log(n))$.

In our approach, the most computationally expensive step is the calculation of pairwise ultrametric distance based on Equation \ref{eq: distance} since it requires finding the lowest common ancestor of two leaf nodes within a tree.
The complexity of finding the lowest common ancestor is $O(h)$, where $h$ is the height of the tree 
(length of longest path from a leaf to the root).
In the worst case $O(h)$ is equivalent to $O(n)$, but if the tree is balanced, $O(log(n))$ can be achieved.
It also requires $O(h)$ space.
Fast algorithm exists that can provide constant-time queries of lowest common ancestor by first processing a tree in linear time.

For large datasets, one way to speed up the computation is pre-cluster the data points into $k$ clusters by 
either leveraging external ontological information (cutting the tree at high levels) or by using a partition-based clustering algorithm. 
Each of the $k$ clusters is then treated separately, and single-link hierarchical clustering algorithm is employed
to build a dendrogram for each sub-cluster.
Finally, the $k$ dendrograms are combined into one dendrogram by applying single-link algorithm which treats each of the $k$ dendrograms as an internal node.
The overall complexity in this case is $O(k(\frac{n}{k})^2log(\frac{n}{k}) + k^2log(k))$.
For reasonably large $k$, the computation time can be greatly reduced.
Within each sub-cluster, the search for optimal $\alpha$ is conducted in a cross-validation manner by evaluating a task-specific metric.
The full algorithm including hyper-parameter tuning is presented in Algorithm \ref{alg: one}.

\section{Case Study: A Customer Behavior Based Product Taxonomy}
\label{sec: experiment}
In this session, we apply the proposed method to the construction of a customer behavior based product taxonomy for an Amazon service.
The goal here is to build a taxonomy that captures substitution effects among different products and product groups.

To achieve this goal, 
we could define a dissimilarity measure between products based on a customer behavior metric,
and group products using a hierarchical clustering algorithm.
However, due to the huge size of Amazon selection and customer base,
customer behavior data is usually sparse and noisy.
Furthermore, for taxonomy building purpose,
we'd like the grouping to be consistent across all levels, 
and the resulting hierarchy to be logical as perceived by a human reader.
As discussed in the introduction,
clustering with only a dissimilarity measure is an ill-posed problem.
It's hard for a clustering algorithm to recover the data partitions that satisfies various criteria of a concrete task.
On the other hand, human-designed taxonomies usually perform well in terms of consistency and human readability. 
In this work, we employ a semi-supervised approach for the building of a product taxonomy,
leveraging the ontological information from existing Amazon-wide browse hierarchy.

\subsection{Amazon browse hierarchy}
Amazon Browse enables customers' discovery experience by organizing Amazon's product selection into a discovery taxonomy.
The browse hierarchy is loaded every time a customer visits Amazon website. 
The leaf nodes of Amazon browse hierarchy represent a group of products of the same type such as coffee-mug, dvd-player etc.
The internal nodes represent higher levels of product groupings. 
% Figure \ref{fig:mytree} presents a small segment of Amazon browse tree under Grocery and Gourmet Food / Beverages.
While being important in influencing customer searches,
Amazon browse trees are not built to reflect program-specific product substitution effects.
They determines what customer see but not their following decisions after seeing the search results.

%\begin{figure}[h]
%	\centering
	%\begin{minipage}{0.8\textwidth}
%		\includegraphics[width=0.52\textwidth]{mytree.png}
%		\caption{\texttt{Amazon browse tree under Grocery and Gourmet Food / Beverages}}
%		\label{fig:mytree}
	%\end{minipage}
%\end{figure}

Due to the huge size of Amazon browse hierarchy, 
we pre-clusterd the data into segments as in Algorithm \ref{alg: one}.
Pariwise distance between leaf nodes are the calculated based on Equation \ref{eq: distance} to incorporate the ontological structure of the browse hierarchy.

\subsection{Customer behavior based dissimilarity measure}
To construct a customer behavior based dissimilarity measurement between leaf nodes,
we first use Latent Dirichlet Allocation (LDA) \cite{Blei2003, BleiDavidCarinLawrence2010} to obtain an embedding for each leaf node based on
customers' click, cart-add and purchase actions for the Amazon service.
To apply LDA to customer searches,
we treat each search keyword as a document, 
and each leaf node as a word in the vocabulary.
Each element in the document-word matrix stores the frequency of certain customer actions such as clicks, cart-adds, purchases
performed on a particular leaf node within the context that customer search for a given keyword.
Provided with the number of topics,
LDA outputs the probability of word appears in each topic.
We use the vector of topic probabilities for each leaf node as the embedding. 
LDA essentially is used here as a dimensionality reduction method similar to matrix factorization.

We then calculate the cosine dissimilarity between pairs of leaf nodes using the embeddings.
Since each element in the embedding is a probability, a positive number, the cosine dissimilarity is between $0$ and $1$. 
The cosine dissimilarity between two leaf nodes $x_{i}$ and $x_{j}$, is calculated as:
\begin{equation}
	\begin{aligned}
		d_{cosine} (x_{i}, x_{j}) = 1 - x_{i} \cdot x_{j} / (\|x_{i}\|_{2} \cdot \|x_{j}\|_{2})
	\end{aligned}
	\label{eq:cosinedistance}
\end{equation}

\subsection{Hyper-parameter tuning by maximizing the performance of substitution group}
Given the problem-specific distance measure,
and the ultrametric distance encoding Amazon browse node hierarchy,
we can combine the two components to form the new distance measure in our semi-supervised hierarchical clustering problem.
As discussed in section \ref{sec: method},
the weighting parameter $\alpha$ can be tuned in a cross-validation manner by optimizing a task-specific performance metric.

\begin{figure*}[h]
	\centering
	%\begin{minipage}{0.8\textwidth}
	%\begin{minipage}{.4\textwidth}
		\includegraphics[width=0.8\textwidth]{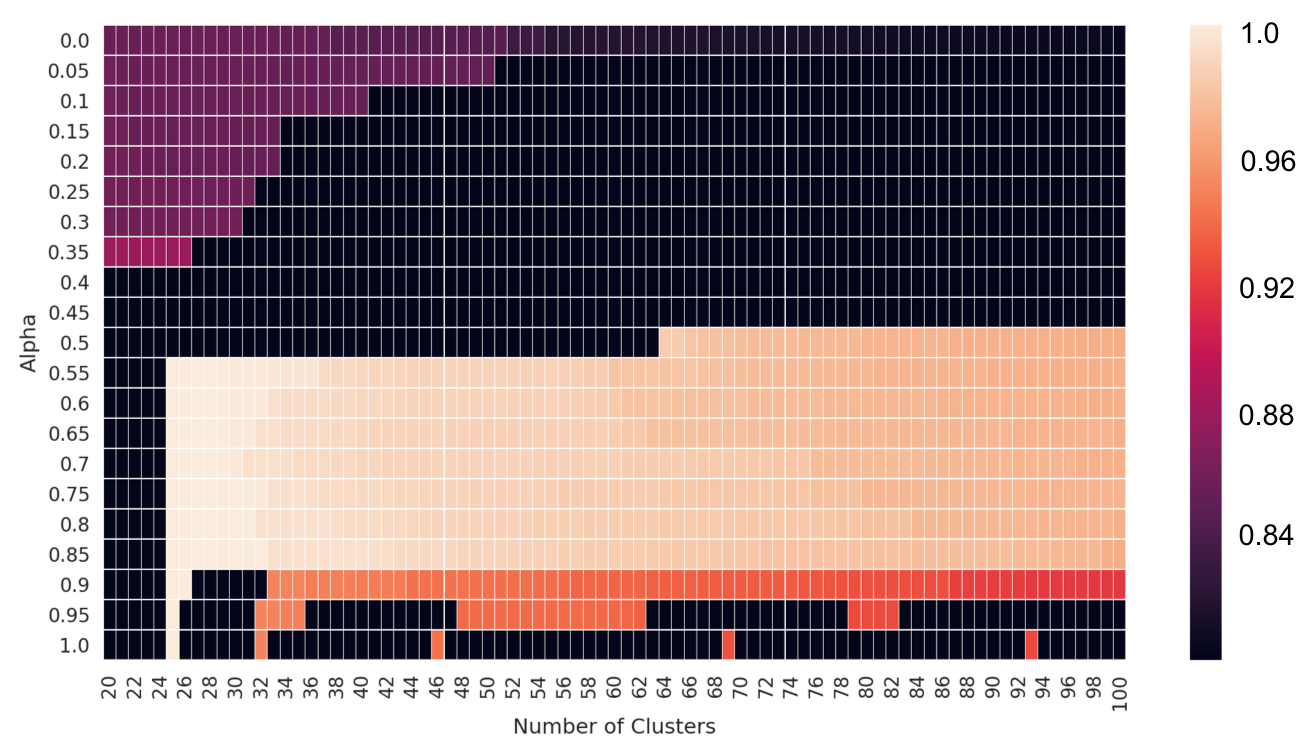}
		%\captionsetup{justification=justified}
		
		\caption{\texttt{Heatmap of cross-validation result for normalized purity metric. 
		              The lighter the color, the higher the purity. 
			      Due to the discrete nature of the tree structure, 
			      certain numbers of clusters are not selectable, 
			      shown as black blocks in the heatmap.}}
		\label{fig:performance}
	%\end{minipage}	
	%\end{minipage}
\end{figure*} 

For evaluation, we optimize the performance of using the resulting clusters as substitution groups,
within which products are substitutable with each other.
It's reasonable to assume that customers who search for the same keyword share similar type of demand.
If all the customers search for the same keyword end up purchasing items from the same substitution group,
then our definition of the substitution group captures all the substitution effect for that demand.
If customers search for the same keyword end up purchasing items from the many different substitution groups,
then our grouping of products does a poor job in capturing product substitution. 
Based on the above rationale, we define three metrics to capture of the substitution performance.
$1.$ \textquotedblleft Purity\textquotedblright, 
which is defined as the average percentage of customer purchases falling within the top substitution group for each search keyword.
$2.$ \textquotedblleft Entropy\textquotedblright, 
for each search keyword, there is a categorical distribution of customer purchases from different substitution groups.
The average entropy of the categorical distribution for each search keyword defines the entropy metric.
$3.$ \textquotedblleft Weighted entropy\textquotedblright,
this metric is similar to the Entropy metric except that each keyword is weighted by the number of customer purchases.
For Purity metric, high values are preferred.
For Entropy metrics, low values are better. 

Based on the performance metrics, 
our experiment was conducted as follows:
A full month of customers' search data was used as the training data to obtain the LDA embedding for the leaf nodes.
A grid search of hyper-parameter $\alpha$ and the number of clusters $K$ was conducted using cross-validation on the data from the first half of the following month.
Figure \ref{fig:performance} shows the cross-validation result as a heat map of the normalized purity metric.
The lighter the color, the higher the purity.
Due to the discrete nature of the tree structure, 
certain numbers of flat clusters can not be formed from the dendrograms.
Those cases are shown as black squares in the heatmap.
As we can see from the figure,
our semi-supervised approach achieves consistently better performance than both
the pure customer behavior based dissimilarity $(\alpha = 0)$ and pure browse taxonomy $(\alpha = 1)$.
Similar trends can be observed for entropy-based metrics (not shown in this paper).
It can be noted from the figure that using the pure browse structure based taxonomy is not flexible in terms of number of clusters.
By mixing the two distance components, we can create hierarchy of leaf nodes at different levels of granularity.
Based on the cross-validation result, we select the best $\alpha$ and test it on the data from the second half of the month.  
The test result is presented in Table \ref{table:1}.
To facilitate the comparison with pure browse node based taxonomy, we choose the cluster numbers of 46 and 69 for testing.
As we can see from the table, the semi-supervised approach performs best during the testing period across all three metrics
(highest in Purity, lowest in Entropy metrics).

\begin{table}[h!]
\centering
\begin{tabular}{||c c c c c||} 
 \hline
Clusters & $\alpha$ & Purity & Entropy & Weighted Entropy \\ [0.4ex] 
 \hline
 46 & 0.0 & 0.93 & 1.0 & 1.0 \\ 
 46 & 0.85 & 1.0 & 0.68 & 0.72 \\
 46 & 1.0 & 0.96 & 0.72 & 0.80 \\
 69 & 0.0 & 0.92 & 1.0 & 1.0 \\
 69 & 0.7 & 1.0 & 0.69 & 0.71 \\
 69 & 1.0 & 0.96 & 0.77 & 0.79 \\[0.2ex] 
 \hline
\end{tabular}
\caption{Testing results on substitution performance (Purity values are normalized against best performance.
Entropy values are normalized against worst performance.)}
\label{table:1}
\end{table}
 
Figure \ref{fig:coffee_tea} presents the evolution of dendrogram structure 
for the segment of \textquotedblleft Coffee, Tea and Cocoa\textquotedblright.
As we decrease $\alpha$ (increase mixing), one can observe mixing of coffee and tea at lower level of the dendrograms,
which reflects a notion of substitution between the two product groups.
In another example, figure \ref{fig:rice_bean} presents the dendrogram evolution for Beans, Grains and Rice segment.
In that case, we can observe a finer grouping of products within either rice group or beans group as we decrease $\alpha$.
However, products from different groups don't mix,
which means the substitution effect is not as significant as that between Coffee and Tea products.

\begin{figure*}[!htb]
	\centering
	%\begin{minipage}{0.8\textwidth}
		\includegraphics[width=0.75\textwidth]{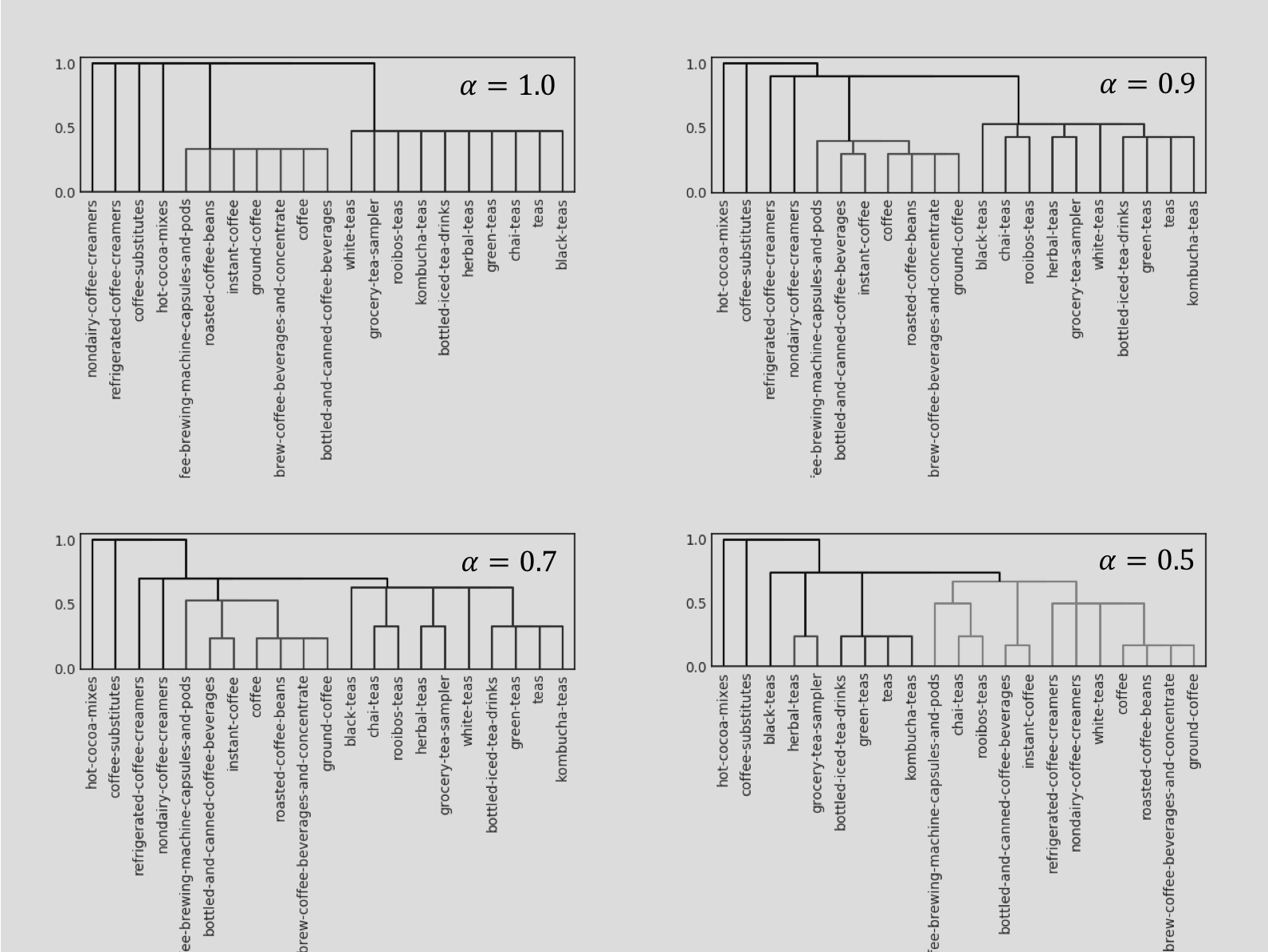}
		\caption{\texttt{Structure evolution of dendrograms for Coffee, Tea and Cocoa segment}}
		\label{fig:coffee_tea}
	%\end{minipage}
\end{figure*}  

\begin{figure*}[!htb]
	\centering
	%\begin{minipage}{0.8\textwidth}
		\includegraphics[width=0.75\textwidth]{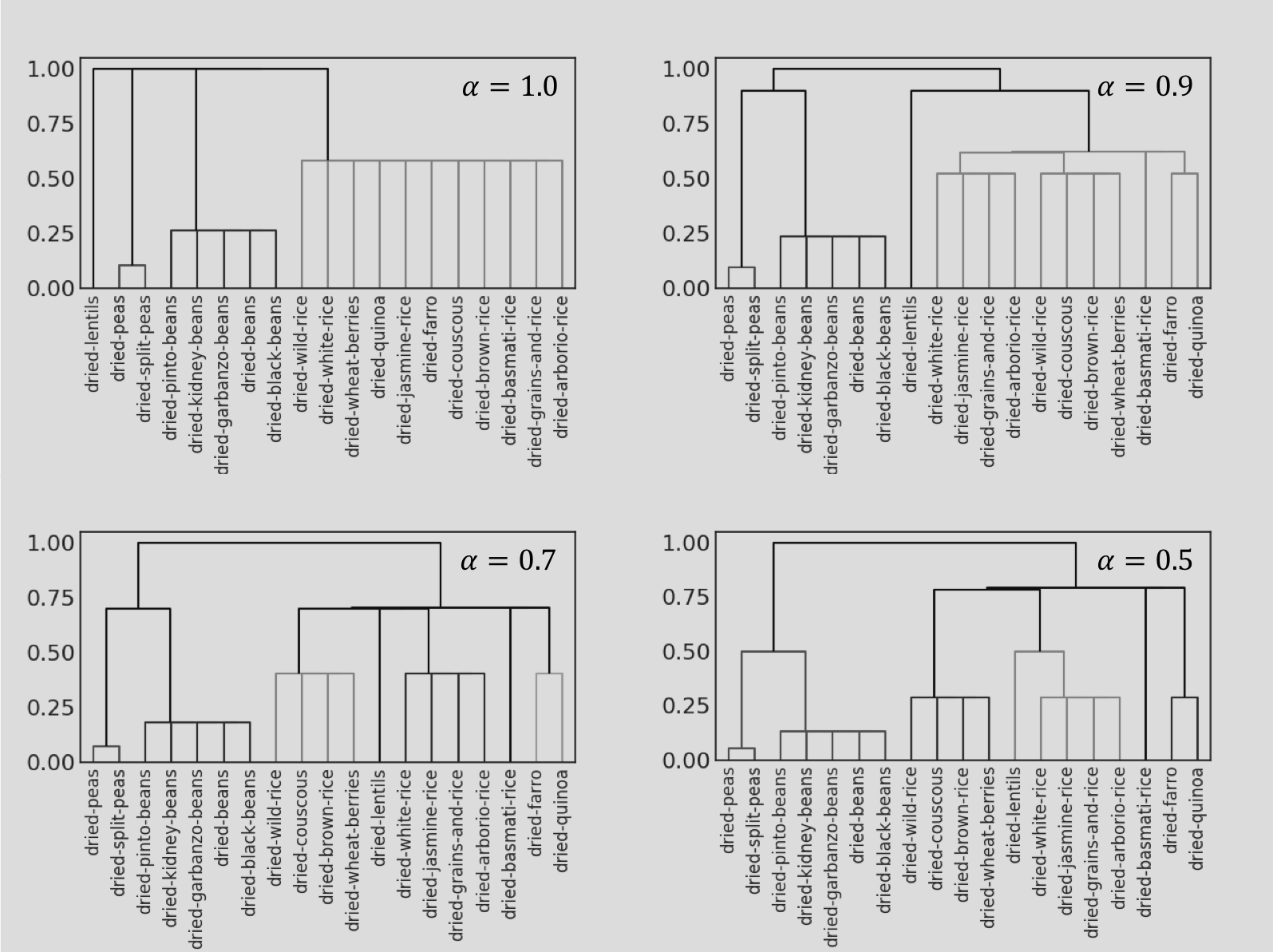}
		\caption{\texttt{Structure evolution of dendrograms for Beans, Grains and Rice segment}}
		\label{fig:rice_bean}
	%\end{minipage}
\end{figure*}  

\section{Conclusion}
\label{sec: conclusion}
Hierarchical clustering is a a prominent class of clustering algorithms.
It has been the dominant approach to constructing embedded classification schemes. 
In this paper, we propose a novel method of incorporating prior domain knowledge 
about entity relations into hierarchical clustering.
By encoding the prior relational information using an ultrametric distance function,
we have shown that the popular linkage based hierarchical clustering algorithms can faithfully
recover the prior relational structure between entities.
We construct the semi-supervised clustering problem by applying the ultrametric distance
as a penalty term to the original task-specific distance measure.
We choose to use single link algorithm to solve the problem due to its favorable stability and convergence properties. 
As an example, we apply the proposed method to the construction of a customer behavior based product taxonomy for an Amazon service leveraging an Amazon-wide browse structure.  
Our experiment results show that the semi-supervised approach achieves better performance than the clustering
purely based on task-specific distance and the clustering purely based on external ontological structure.

\appendix
\section{Complete linkage hierarchical clustering with ultrametric distance}
\label{appendix: completelink}
It is known that in a metric space, when there are two or more equally good candidates for merging at a certain step, 
the results from complete link hierarchical clustering algorithms depend on the ordering of merging.
In this section, we show that if the distance function is ultrametric, 
the dendrogram structure from complete linkage does not depend on the merging order.

\begin{proof}
We first show under complete link and ultrametric assumptions,
the ultrametric condition also holds among clusters.
Let $a, b, c$ represent three disjoint clusters (can be singletons), 
we want to show $D(c, a) \leq max(D(a, b), D(b, c))$.

Under complete linkage, without loss of generality, we assume 
$x_{1}, x_{6} \in a, x_{2}, x_{3} \in b, x_{4}, x_{5} \in c$, and
\begin{equation}
	\begin{aligned}
		D(a, b) & = {\underset{x \in a, x' \in b} {max} u(x, x')} = u(x_{1}, x_{2}) \\
		D(b, c) & = {\underset{x \in b, x' \in c} {max} u(x, x')} = u(x_{3}, x_{4}) \\
		D(c, a) & = {\underset{x \in c, x' \in a} {max} u(x, x')} = u(x_{5}, x_{6})
	\end{aligned}
	\label{eq:proof1}
\end{equation}
Then we have, 
\begin{equation}
	\begin{aligned}
		max(D(a, b), D(b, c)) & = max(u(x_{1}, x_{2}), u(x_{3}, x_{4}) \\
						& \geq max(u(x_{6}, x_{2}), u(x_{2}, x_{5}) \\
						& \geq u(x_{5}, x_{6}) \\
						& = D(c, a)
	\end{aligned}
	\label{eq:proof2}
\end{equation}

We now show for any disjoint clusters $a, b, c, d$,
if at a certain stage $D(a, b) = D(b, c)$ are smaller than other cluster-cluster distances,
which means $(a, b)$ and $(b, c)$ are equally good candidates for merge next.
Regardless of merging order between $(a, b)$ and $(b, c)$, 
cluster $d$ will always merge last. 

In fact, due to ultrametric condition, if $D(a, b) = D(b, c)$,
then $max(D(a, b), D(b, c)) = D(a, b) = D(b, c) \geq D(a, c)$.
It means $(a, c)$ will merge before $(a, b)$ or $(b, c)$.
Then, there is no ambiguity about merging order.
$(a, c)$ merges first, then $(ac, b)$. 
$d$ will always merge last to the cluster.
\end{proof}

In a similar manner, we can show the same result for average link hierarchical clustering with ultrametric distance.

\bibliographystyle{ACM-Reference-Format}
\bibliography{clustering}

\end{document}